\documentclass{article} 
\usepackage{nips10submit_e,times}

\usepackage{algorithm2e,framed}
\usepackage{amsfonts,amsmath}
\usepackage{graphicx}
\usepackage{url}

\usepackage{amsthm,amsmath,amssymb,graphicx,epsfig,color}
\usepackage{fullpage}

\newcommand{\setlinespacing}[1]%
           {\setlength{\baselineskip}{#1 \defbaselineskip}}

\newcommand{\BlockDiagk}[1]{\mbox{}\left(%
\begin{array}{cc}
  \Sigma_{k} & \bf{0} \\
  \bf{0} &  \Sigma_{\rho-k}\\
\end{array}\right)}

\newcommand{\BlockDiagkk}[1]{\mbox{}\left(%
\begin{array}{cc}
  \Sigma_{k} & \bf{0} \\
  \bf{0} & \bf{0} \\
\end{array}\right)}

\newcommand{\BlockDiagkrk}[1]{\mbox{}\left(%
\begin{array}{cc}
  \bf{0} & \bf{0} \\
  \bf{0} & \Sigma_{\rho-k} \\
\end{array}\right)}

\newcommand{\BlockDiagkkh}[1]{\mbox{}\left(%
\begin{array}{c}
  \Sigma_{k} \\
  \bf{0} \\
\end{array}\right)}

\newcommand{\BlockDiagkrkh}[1]{\mbox{}\left(%
\begin{array}{c}
  \bf{0} \\
  \Sigma_{\rho-k} \\
\end{array}\right)}

\long\def\killtext#1{}

\newcommand{\eps}{\varepsilon}

\newcommand{\ignore}[1]{}

\newcommand{\Prob}[1]{\ensuremath{\mathbb{P}\left(#1\right)}}

\newcommand{\EE}[1]{\ensuremath{\mathbb{E}\left[#1\right] } }

\newcommand{\var}[1]{\text{Var}\ensuremath{\left[#1\right] } }

\newcommand{\pinv}[1]{ {#1}^\dagger}

\newcommand{\norm}[1]{\ensuremath{\left\|#1\right\|_2}}

\newcommand{\frobnorm}[1]{\ensuremath{\left\|#1\right\|_{\text{\rm F}}}}

\def\RR  {\mathbb{R}}

\newtheorem{definition}{Definition}
\newtheorem{theorem}{Theorem}
\newtheorem{proposition}[theorem]{Proposition}
\newtheorem{lemma}[theorem]{Lemma}

\usepackage{hyperref}

\title{Random Projections for $k$-means Clustering}

\author{
Christos Boutsidis  \\
Department of Computer Science\\
RPI\\
\And
Anastasios Zouzias  \\
Department of Computer Science\\
University of Toronto\\
\And
Petros Drineas  \\
Department of Computer Science\\
RPI\\
}

\nipsfinalcopy 

\begin{document}

\maketitle

\begin{abstract}
This paper discusses the topic of dimensionality reduction for
$k$-means clustering. We prove that any set of $n$ points in $d$
dimensions (rows in a matrix $A \in \RR^{n \times d}$) can be projected into $t =
\Omega(k / \eps^2)$ dimensions, for any $\eps \in (0,1/3)$, in
$O(n d \lceil \eps^{-2} k/ \log(d) \rceil )$ time, such that with 
constant probability the optimal $k$-partition of the point
set is preserved within a factor of $2+\eps$. The projection is
done by post-multiplying $A$ with a $d \times t$ random
matrix $R$ having entries $+1/\sqrt{t}$ or $-1/\sqrt{t}$ with equal probability.
A numerical implementation of our technique and experiments on a large face images dataset
verify the speed and the accuracy of our theoretical results. 
\end{abstract}

\section{Introduction}
The $k$-means clustering algorithm~\cite{Llo82} was recently
recognized as one of the top ten data mining tools of the
last fifty years \cite{Wu07}. In parallel, random projections (RP)
or the so-called Johnson-Lindenstrauss type embeddings~\cite{JL84}
became popular and found applications in both theoretical computer
science \cite{AC06} and data analytics~\cite{BM01}. This paper
focuses on the application of the random projection method 
(see Section \ref{sec:rp}) to the
$k$-means clustering problem (see Definition \ref{def:kmeans}).
Formally, assuming as input a set of $n$ points in $d$ dimensions,
our goal is to randomly project the points into $\tilde{d}$ dimensions,
with $\tilde{d} \ll d$, and then apply a $k$-means clustering algorithm 
(see Definition \ref{def:approx}) on
the projected points. Of course, one should be able to
compute the projection fast without distorting significantly 
the ``clusters'' of the original point set. Our algorithm (see Algorithm
\ref{alg:selection}) satisfies both conditions by computing the
embedding in time linear in the size of the input and by distorting
the ``clusters'' of the dataset by a factor of at most $2
+ \eps$, for some $\eps \in (0,1/3)$ (see Theorem
\ref{thm:second_result}). We believe that the high dimensionality of
modern data will render our algorithm useful and attractive in many practical applications \cite{GGBD05}.

Dimensionality reduction encompasses the union of two different
approaches: \emph{feature selection}, which embeds the points into a
low-dimensional space by selecting actual dimensions of the data,
and \emph{feature extraction}, which finds an embedding by constructing
new artificial features that are, for example, linear combinations
of the original features. Let $A$ be an $n\times d$ matrix containing $n$ $d$-dimensional points ($A_{(i)}$ denotes the $i$-th point of the set), and let $k$ be the
number of clusters (see also Section \ref{sxn:notation} for more notation). We slightly abuse notation by also denoting by $A$ the $n$-point set formed by the rows of $A$. We say that an embedding $f: A \to \RR^{\tilde{d}}$ with $f(A_{(i)})=\tilde{A}_{(i)}$ for all $i\in{[n]}$ and some
$\tilde{d} < d$, preserves the clustering structure of $A$ within a factor
$\phi$, for some $\phi \geq 1$, if finding an optimal clustering in $\tilde{A}$
and plugging it back to $A$ is only a factor of $\phi$ worse than finding the optimal
clustering directly in $A$. Clustering optimality and approximability are formally presented in
Definitions \ref{def:kmeans} and \ref{def:approx}, respectively. 
Prior efforts on designing provably accurate dimensionality
reduction methods for $k$-means clustering include:
\textit{(i)} the Singular Value Decomposition (SVD), where one finds an
embedding with image $\tilde{A} = U_k \Sigma_k \in \RR^{n \times k}$ such that
the clustering structure is preserved within a factor of two; \textit{(ii)} random projections, where
one projects the input points into $ t = \Omega(\log(n)/ \eps^2)$ dimensions such
that with constant probability the clustering structure is preserved within a factor of $1+\eps$ (see Section~\ref{sec:rp}); 
\textit{(iii)} SVD-based feature selection,
where one can use the SVD to find $c = \Omega(k
\log(k/\eps) / \eps^2)$ actual features, i.e. an embedding with image $\tilde{A}\in \RR^{n \times c}$ containing (rescaled) columns
from $A$, such that with constant probability the clustering structure is preserved within a factor of
$2+\eps$. These results are summarized in Table \ref{table:all}. 
\begin{table}[htdp]
\begin{center}
\begin{tabular}{|c|c|c|c|c|c|}
\hline
\textbf{Year} & \textbf{Ref.} & \textbf{Description}              & \textbf{Dimensions}                              & \textbf{Time} & \textbf{Accuracy}          \\
\hline
1999 & \cite{DFKVV99}   & SVD - feature extraction & $k$                                     & $O(nd\min\{n,d\})$ & $2$          \\
\hline
 - & Folklore    & RP - feature extraction  & $\Omega(\log(n)/\eps^2)$            & $O(nd \lceil \eps^{-2} \log(n)/ \log(d) \rceil)$   & $1+\eps$ \\
\hline
2009 & \cite{BMD09b}    & SVD - feature selection  & $\Omega(k \log(k/\eps)/\eps^2)$ & $O(nd\min\{n,d\})$ & $2+\eps$ \\
\hline
2010 & This paper    & RP - feature extraction  & $\Omega(k / \eps^2)$                & $O(n d \lceil \eps^{-2} k/ \log(d) \rceil)$   & $2+\eps$ \\
\hline
\end{tabular}
\end{center}
\label{table:all} \caption{\small{Dimension reduction methods for
$k$-means. In the RP methods the construction is done with random sign matrices and the mailman algorithm 
(see Sections \ref{sec:rp} and \ref{sec:rt}, respectively)}.
}
\end{table}
A
head-to-head comparison of our algorithm with existing results
allows us to claim the following improvements: 
\textit{(i)} reduce the
running time by a factor of $\min\{n,d\} \lceil \eps^2\log(d) /k \rceil$, while losing only a
factor of $\eps$ in the approximation accuracy and a factor of
$1/\eps^2$ in the dimension of the embedding; \textit{(ii)} reduce
the dimension of the embedding and the running time by a factor of
$\log(n)/k$ while losing a factor of one in the approximation accuracy; \textit{(iii)} reduce the dimension of the embedding by a factor
of $\log(k/\eps)$ and the running time by a factor of $\min\{n,d\} \lceil \eps^2\log(d) /k \rceil$, 
respectively. Finally, we should point out that other 
techniques, for example the Laplacian scores~\cite{HCN06} or the Fisher scores~\cite{FS75}, 
are very popular in applications (see also surveys on the topic~\cite{GE03,KCS10}). However, they 
lack a theoretical worst case analysis of the form we describe in this work.

\section{Preliminaries} \label{sec:intro}
We start by formally defining the $k$-means clustering problem using matrix notation. Later in this section, we precisely describe the approximability framework adopted 
in the $k$-means clustering literature and fix the notation.
\begin{definition} \label{def:kmeans}
\textsc{[The k-means clustering problem]}\\
Given a set of $n$ points in $d$ dimensions (rows in an $n \times
d$ matrix $A$) and a positive integer $k$ denoting the number of
clusters, find the $n \times k$ indicator matrix $X_{opt}$ such
that
\begin{equation}
\label{eqn:def}
 X_{opt} = \arg \min_{X \in \cal{X}} \frobnorm{A - X X^\top  A}^2.
\end{equation}
\end{definition}
Here $\cal{X}$ denotes the set of all $n \times k$
indicator matrices $X$. The functional $F(A,X) = \frobnorm{A - X
X^\top  A}^2$ is the so-called $k$-means objective function. 
An $n \times k$ indicator matrix has exactly one non-zero element per row,
which denotes cluster membership. Equivalently, for all
$i=1,\ldots,n$ and $j=1,\ldots,k$, the $i$-th point
belongs to the $j$-th cluster if and only if $X_{ij}=1/\sqrt{z_j}$, where $z_j$ denotes the number of
points in the corresponding cluster. Note that $X^\top X = I_k$, where $I_k$ is the
$k\times k$ identity matrix. 

\subsection{Approximation Algorithms for $k$-means clustering}
Finding $X_{opt}$ is an NP-hard problem even for
$k=2$~\cite{ADHP09}, thus research has focused on developing
approximation algorithms for $k$-means clustering. The following
definition captures the framework of such efforts.
\begin{definition} \label{def:approx}
\textsc{[k-means approximation algorithm]}\\
An algorithm is a ``$\gamma$-approximation'' for the $k$-means
clustering problem ($\gamma \geq 1$) if it takes inputs $A$ and
$k$, and returns an indicator matrix $X_{\gamma}$ that satisfies
with probability at least $1 - \delta_{\gamma}$,
\begin{equation}
\frobnorm{A - X_{\gamma} X_{\gamma}^\top  A}^2 \leq \gamma \min_{X
\in \cal{X}} \frobnorm{A - X X^\top  A}^2.
\end{equation}
In the above, $\delta_{\gamma} \in [0,1)$ is the failure
probability of the $\gamma$-approximation $k$-means algorithm.
\end{definition}
For our discussion, we fix the $\gamma$-approximation algorithm to be the one presented in~\cite{KSS04}, which guarantees 
$\gamma =1+\eps'$ for any $\eps' \in (0, 1]$ with running time $O( 2^{(k/\eps')^{O(1)}} d n )$. 
\subsection{Notation} \label{sxn:notation}
Given an $n \times d$ matrix $A$ and an integer $k$ with $k < \min\{n,d\}$, let $U_k \in
\mathbb{R}^{n \times k}$ (resp. $V_k \in \mathbb{R}^{d \times k}$)
be the matrix of the top $k$ left (resp. right) singular vectors
of $A$, and let $\Sigma_k \in \mathbb{R}^{k \times k}$ be a
diagonal matrix containing the top $k$ singular values of $A$ in
non-increasing order. If we let $\rho$ be the rank of $A$, then
$A_{\rho-k}$ is equal to $A - A_k$, with $A_k = U_k\Sigma_k
V_k^\top $. By $A_{(i)}$ we denote the $i$-th row of $A$. For an index $i$ taking values in the set $\{1,
\ldots, n\}$ we write $i \in [n]$. We denote, in non-increasing
order, the non-negative singular values  of $A$ by $\sigma_i(A)$
with $i \in [\rho]$. $\frobnorm{A}$ and $\norm{A}$ denote the
Frobenius and the spectral norm of a matrix $A$, respectively.
$\pinv{A}$ denotes the pseudo-inverse of $A$, i.e. the unique $d \times
n$ matrix satisfying $A = A \pinv{A} A$, $\pinv{A} A \pinv{A} = \pinv{A}$,
$(A \pinv{A} )^\top = A \pinv{A}$, and $(\pinv{A} A )^\top = \pinv{A} A$.
Note also that
$\norm{\pinv{A}} = \sigma_1(\pinv{A}) = 1 / \sigma_{\rho}(A)$ and
$\norm{A} = \sigma_1(A) = 1 / \sigma_{\rho}(\pinv{A})$. A useful
property of matrix norms is that for any two matrices $C$ and $T$
of appropriate dimensions, $\frobnorm{CT} \leq \frobnorm{C}
\norm{T}$; this is a stronger version of the standard
submultiplicavity property. We call $P$ a projector matrix if it
is square and $P^2=P$. We use $\EE{Y}$ and $\var{Y}$ to take the
expectation and the variance of a random variable $Y$ and
$\Prob{e}$ to take the probability of an event $e$. We abbreviate
``independent identically distributed'' to ``i.i.d.'' and ``with
probability'' to ``w.p.''. Finally, all logarithms are base two.  
\subsection{Random Projections}\label{sec:rp}
A classical result of Johnson and Lindenstrauss states that any $n$-point set in $d$ dimensions - rows in a matrix $A \in \RR^{n \times d}$ - can be 
linearly projected into $t=\Omega(\log (n) /\eps^2)$ dimensions while preserving pairwise distances within a factor of $1\pm\eps$ using a 
random orthonormal matrix~\cite{JL84}. Subsequent research simplified the proof of the above result by showing that such a projection can be generated using a $d \times t$ random Gaussian 
matrix $R$, i.e., a matrix whose entries are i.i.d. Gaussian random variables with zero mean and variance $1/\sqrt{t}$~\cite{IM:curse_dim}. More precisely, the following inequality holds with high probability over the randomness of $R$,
\begin{equation}\label{eqn:rp}
(1- \eps) \norm{ A_{(i)} - A_{(j)} } \leq  \norm{ A_{(i)}R - A_{(j)}R } \leq  (1 + \eps) \norm{ A_{(i)} - A_{(j)}}.
\end{equation}
Notice that such an embedding $\tilde{A} = A R$ preserves 
the metric structure of the point-set, so it also preserves, within a factor of $1+\eps$, the optimal value of the $k$-means objective function of $A$. 
Achlioptas proved that even a (rescaled) random sign matrix suffices in order to get the same guarantees as above~\cite{Ach03}, an approach that we adopt here (see step two in 
Algorithm \ref{alg:selection}). Moreover, in this paper we will heavily exploit the structure of such a random matrix, and obtain, as an added bonus, savings on the computation of the projection.
\section{A random-projection-type $k$-means algorithm}
\vspace{-0.1in} \label{sec:algo} Algorithm \ref{alg:selection}
takes as inputs the matrix $A \in \RR^{n\times d}$, the number of
clusters $k$, an error parameter $\eps \in (0,1/3)$, and some
$\gamma$-approximation $k$-means algorithm. It returns an
indicator matrix $X_{\tilde{\gamma}}$  determining a $k$-partition
of the rows of $A$. 
\begin{algorithm}
\begin{framed}
\textbf{Input:} $n \times d$ matrix $A$ ($n$ points, $d$ features), number of clusters $k$, error parameter $\eps\in{(0,1/3)}$, and $\gamma$-approximation $k$-means algorithm. \\
\noindent \textbf{Output:} Indicator matrix $X_{\tilde{\gamma}}$
determining a $k$-partition on the rows of $A$.

\begin{enumerate}

  \item Set $t = \Omega(k / \eps^2)$, i.e. set $t = t_o \geq c k /\eps^2$ for a sufficiently large constant $c$.
\vspace{-0.05in}
  \item Compute a random $d \times t$ matrix $R$ as follows. For all $i \in [d]$, $j \in [t]$
   \[ R_{ij} = \begin{cases}
       +1/\sqrt{t}, \text{w.p. 1/2},\\
      -1/\sqrt{t}, \text{w.p. 1/2}.
\end{cases} \]
\vspace{-0.15in}
   \item Compute the product $\tilde{A} = A R$.
\vspace{-0.05in}
   \item Run the $\gamma$-approximation algorithm on $\tilde{A}$ to obtain
   $X_{\tilde{\gamma}}$; Return the indicator matrix $X_{\tilde{\gamma}}$

\end{enumerate}

\caption{ A random projection algorithm for $k$-means clustering.}
\label{alg:selection}

\end{framed}
\end{algorithm}
%
\subsection{Running time analysis} 
\label{sec:rt}
Algorithm \ref{alg:selection}
reduces the dimensions of $A$ by post-multiplying it with a random sign 
matrix $R$. Interestingly, any ``random projection matrix''
$R$ that respects the properties of Lemma~\ref{lem:sarlos} 
with $t = \Omega(k / \eps^2)$ can be used in this step. If $R$
is constructed as in Algorithm \ref{alg:selection},
one can employ the so-called mailman algorithm for matrix
multiplication~\cite{LZ09} and compute the product $A R$ in $O(n d \lceil \eps^{-2} k/ \log(d) \rceil )$ time. 
Indeed, the mailman algorithm computes (after preprocessing
\footnote{\small{Reading the input $d \times \log d$ sign matrix requires $O(d\log d)$ time. However, in our case we only consider 
multiplication with a \emph{random} sign matrix, therefore we can avoid the preprocessing step by directly computing 
a random \emph{correspondence} matrix as discussed in~\cite[Preprocessing Section]{LZ09}.} })
 a matrix-vector product of any $d$-dimensional vector (row of $A$) with an 
$d \times \log(d)$ sign matrix in $O(d)$ time. 
By partitioning the columns of our $d\times t$ matrix $R$ into $\lceil t/\log(d)\rceil$ blocks, the claim follows. 
Notice that when 
$k= O(\log(d))$, then we get an - almost - linear time complexity $O(nd/\eps^2 )$. 
The latter assumption is reasonable in our
setting since the need for dimension reduction in $k$-means clustering arises usually in high-dimensional data (large $d$). 
Other choices of $R$ would give the same approximation results;
the time complexity to compute the embedding would be different
though. A matrix where each entry is a random Gaussian variable
with zero mean and variance $1/\sqrt{t}$ would imply an $O(k
n d / \eps^2)$ time complexity (naive multiplication).
In our experiments in Section \ref{sec:experiments} we experiment
with the matrix $R$ described in Algorithm \ref{alg:selection} and employ MatLab's 
matrix-matrix BLAS implementation to proceed in the third step of the algorithm. We also experimented with a novel MatLab/C implementation of the 
mailman algorithm but, in the general case, we were not able to outperform MatLab's built-in routines (see section \ref{sec:mail}).

Finally, note that any $\gamma$-approximation
algorithm may be used in the last step of Algorithm
\ref{alg:selection}. Using, for example, the algorithm of~\cite{KSS04} with $\gamma = 1+\eps$ 
would result in an algorithm that preserves the clustering within a factor of $2 + \eps$, for 
any $\eps \in (0,1/3)$, running in time 
$O( n d \lceil \eps^{-2} k/ \log(d) \rceil + 2^{(k/\eps)^{O(1)}} k n / \eps^2 )$. In practice though, the Lloyd
algorithm~\cite{Llo82,ORSS06} is very popular and although it does not admit a 
worst case theoretical analysis, it empirically does well. 
We thus employ the Lloyd algorithm for our experimental evaluation of our algorithm in Section~\ref{sec:experiments}.
Note that, after using the proposed dimensionality reduction method, the cost of the 
Lloyd heuristic is only $O( n k^2 /\eps^2 )$ per iteration. This should be compared to the cost of  
$O( k n d )$ per iteration if applied on the original high dimensional data.
\section{Main Theorem}
Theorem \ref{thm:second_result} is our main
quality-of-approximation result for Algorithm \ref{alg:selection}.
Notice that if $\gamma = 1$, i.e. if the $k$-means problem with
inputs $\tilde{A}$ and $k$ is solved exactly, Algorithm~\ref{alg:selection} 
guarantees a distortion of at most $2+\eps$,
as advertised.
\begin{theorem}\label{thm:second_result}
Let the $n \times d$ matrix $A$ and the positive integer $k< \min\{n, d\}$ be
the inputs of the $k$-means clustering problem. Let $\eps \in
(0,1/3)$ and assume access to a $\gamma$-approximation $k$-means
algorithm. Run Algorithm \ref{alg:selection} with inputs $A$, $k$,
$\eps$, and the $\gamma$-approximation algorithm in order to
construct an indicator matrix $X_{\tilde{\gamma}}$. Then with
probability at least $0.97 - \delta_{\gamma}$,
\begin{eqnarray} \label{eqn:main2}
\frobnorm{A - X_{\tilde{\gamma}} X_{\tilde{\gamma}}^\top  A}^2
\leq \left(1+(1+\eps)\gamma\right) \frobnorm{A - X_{opt}
X_{opt}^\top  A}^2.
\end{eqnarray}
\end{theorem}
\subsection*{Proof of Theorem~\ref{thm:second_result}}
The proof of Theorem \ref{thm:second_result} employs several
results from \cite{Sar06} including Lemma~$6$, $8$ and Corollary~$11$. We summarize these results in Lemma
\ref{lem:sarlos} below. Before employing Corollary $11$, Lemma $6$, and Lemma $8$ from
\cite{Sar06} we need to make sure that the matrix $R$ constructed
in Algorithm \ref{alg:selection} is consistent with Definition $1$
and Lemma $5$ in \cite{Sar06}. Theorem $1.1$ of \cite{Ach03} immediately shows that 
the random sign matrix $R$ of Algorithm \ref{alg:selection} satisfies Definition $1$ and Lemma $5$ in \cite{Sar06}.
\begin{lemma} \label{lem:sarlos}
Assume that the matrix $R$ is constructed by using Algorithm~\ref{alg:selection} with inputs $A$, $k$ and $\eps$.
\begin{enumerate}
    \item Singular Values Preservation: For all $i \in [k]$ and w.p. at least $0.99$,
    $$ |1 - \sigma_i(V_k^\top R)| \leq \eps.$$
    \item Matrix Multiplication: For any
    two matrices $S \in \RR^{n \times d}$  and $T \in \RR^{d \times
    k}$,
    $$ \EE{ \frobnorm{ ST - S R R^\top T}^2   } \leq \frac{2}{t} \frobnorm{S}^2 \frobnorm{T}^2. $$
    \item Moments: For any
    $C \in \RR^{n \times d}$: $ \EE{\frobnorm{CR}^2} = \frobnorm{C}^2$ and $\var{\frobnorm{CR}} \leq  2 \frobnorm{C}^4 / t.$
\end{enumerate}
\end{lemma}
The first statement above assumes $c$ being sufficiently large (see step $1$ of Algorithm~\ref{alg:selection}). We continue with several novel results of general interest.
\begin{lemma} \label{lem:sigma_bound}
Under the same assumptions as in Lemma \ref{lem:sarlos} and w.p. at least $0.99$,
\begin{equation}\label{ineq:pseudo_inv_transpose}
\norm{\pinv{(V_k^\top R)} - (V_k^\top R)^\top }\  \leq\ 3 \eps.
\end{equation}
\end{lemma}
\begin{proof}
Let $\Phi = V_k^\top R$; note that $\Phi$ is a $k \times t$ matrix
and the $SVD$ of $\Phi$ is $\Phi = U_{\Phi} \Sigma_{\Phi}
V_\Phi^\top $, where $U_{\Phi}$ and $\Sigma_{\Phi}$ are $k \times
k$ matrices, and $V_\Phi$ is a $t \times k$ matrix. By taking the
SVD of $\pinv{(V_k^\top R)}$ and $(V_k^\top R)^\top $ we get
\[ \norm{\pinv{(V_k^\top R)} - (V_k^\top R)^\top }\ =\ \norm{ V_{\Phi} \Sigma_{\Phi}^{-1} U_{\Phi}^\top  - V_{\Phi}
\Sigma_{\Phi} U_{\Phi}^\top }\ =\ \norm{ V_{\Phi}(\Sigma_{\Phi}^{-1} - \Sigma_{\Phi}) U_{\Phi}^\top }\ =\ \norm{
\Sigma_{\Phi}^{-1} - \Sigma_{\Phi} },\] since $V_{\Phi}$ and
$U_{\Phi}^\top $ can be dropped without changing any unitarily
invariant norm. Let $\Psi = \Sigma_{\Phi}^{-1} - \Sigma_{\Phi}$;
$\Psi$ is a $k \times k$ diagonal matrix. Assuming that, for all
$i \in [k]$, $\sigma_i(\Phi)$ and $\tau_i(\Psi)$ denote the $i$-th
largest singular value of $\Phi$ and the $i$-th diagonal element
of $\Psi$, respectively, it is
\[ \tau_i(\Psi)\  =\ \frac{ 1 - \sigma_i^2(\Phi)  }{ \sigma_{i}(\Phi) }.\]
Since $\Psi$ is a diagonal matrix,
\[\norm{ \Psi }\ =\ \max_{1 \leq i \leq k} \tau_i(\Psi)\ =\ \max_{1 \leq i \leq k} \frac{ 1 - \sigma_i^2(\Phi)}{ \sigma_{i}(\Phi) }.\]
The first statement of Lemma \ref{lem:sarlos}, our choice of $\eps \in (0,1/3)$, and elementary calculations suffice to conclude the proof.
\end{proof}
\begin{lemma} \label{lem:chebyshev_bound}
Under the same assumptions as in Lemma \ref{lem:sarlos} and for
any $n\times d$ matrix $C$ w.p. at least $0.99$,
\begin{equation}\label{ineq:frob_bound}
\frobnorm{ C R }\ \leq\ \sqrt{(1+\eps)} \frobnorm{C}.
\end{equation}
\end{lemma}
\begin{proof}
Notice that there exists a sufficiently large constant $c$ such that $t \geq c k / \eps^2$.
Then, setting $Z = \frobnorm{ C R }^2$, using the third statement of
Lemma \ref{lem:sarlos}, the fact that $k \geq 1$, and Chebyshev's
inequality we get
\begin{eqnarray*}
\Prob{ |Z - \EE{Z }| \geq \eps \frobnorm{C}^2 }\ \leq\ \frac{\var{Z}}{\eps^2 \frobnorm{C}^4}\
    \leq\ \frac{2\frobnorm{C}^4}{t\eps^2 \frobnorm{C}^4}\ \leq\ \frac{2}{ck} \leq\ 0.01.
\end{eqnarray*}
The last inequality follows assuming $c$ sufficiently large.
Finally, taking square root on both sides concludes the proof.
\end{proof}
\begin{lemma} \label{lem:decomposition}
Under the same assumptions as in Lemma \ref{lem:sarlos} and w.p.
at least $0.97$,
\begin{equation}\label{eq:matrix_decomp}
A_k  = (AR) \pinv{(V_k^\top  R)}V_k^\top  + E,
\end{equation}
where $E$ is an $n \times d $ matrix with $\frobnorm{E} \leq 4
\eps \frobnorm{A-A_k}$.
\end{lemma}
\begin{proof}
Since $(AR) \pinv{(V_k^\top  R)}V_k^\top $ is an $n \times d$
matrix, let us write $E = A_k  - (AR) \pinv{(V_k^\top  R)}V_k^\top
$. Then, setting $A = A_k + A_{\rho-k}$, and using the triangle
inequality we get
\[ \frobnorm{E}\ \leq\ \frobnorm{A_k - A_kR \pinv{(V_k^\top R)} V_k^\top }\ +\ \frobnorm{ A_{\rho-k}R \pinv{(V_k^\top R)} V_k^\top  }.\]
The first statement of Lemma \ref{lem:sarlos} implies that
$\text{rank}(V_k^\top R) = k$ thus $(V_k^\top R) \pinv{(V_k^\top R)} =
I_k$, where $I_k$ is the $k \times k$ identity matrix. Replacing
$A_k = U_k \Sigma_k V_k^\top $ and setting $(V_k^\top R)
\pinv{(V_k^\top R)} = I_k$ we get that
\[ \frobnorm{A_k - A_kR \pinv{(V_k^\top R)} V_k^\top }\ =\ \frobnorm{A_k - U_k\Sigma_k V_k^\top R \pinv{(V_k^\top R)} V_k^\top }\ =\ \frobnorm{A_k - U_k \Sigma_k V_k^\top }\ =\ 0. \]
To bound the second term above, we drop $V_k^\top $, add and
subtract the matrix $A_{\rho-k}R (V_k^\top R)^\top  V_k^\top $,
and use the triangle inequality and submultiplicativity:
\begin{eqnarray*}
\frobnorm{ A_{\rho-k} R\pinv{(V_k^\top R)} V_k^\top  } & \leq & \frobnorm{ A_{\rho-k}R (V_k^\top R)^\top  }\ +\ \frobnorm{ A_{\rho-k}R ( \pinv{(V_k^\top R)} - (V_k^\top R)^\top )} \\
                                               & \leq & \frobnorm{ A_{\rho-k}R R^\top  V_k  }\ +\ \frobnorm{ A_{\rho-k}R} \norm{ \pinv{(V_k^\top R)} - (V_k^\top R)^\top
                                               }.
\end{eqnarray*}
Now we will bound each term individually. A crucial observation
for bounding the first term is that $A_{\rho-k}V_k=
U_{\rho-k}\Sigma_{\rho-k}V_{\rho-k}^{\top}V_k=\mathbf{0}$ by
orthogonality of the columns of $V_k$ and $V_{\rho-k}$. This term
now can be bounded using the second statement of Lemma
\ref{lem:sarlos} with $S = A_{\rho - k}$ and $T = V_k$. This statement, assuming $c$ sufficiently large, and an application of Markov's inequality on the random variable
$\frobnorm{A_{\rho-k}R R^\top  V_k  - A_{\rho-k}V_k}$ give that
w.p. at least $0.99$,
\begin{eqnarray}
    \frobnorm{A_{\rho-k}R R^\top  V_k }\ \leq\ 0.5 \eps \frobnorm{A_{\rho-k}}. \label{ineq:mm_frob}
\end{eqnarray}
The second two terms can be bounded using
Lemma~\ref{lem:sigma_bound} and Lemma~\ref{lem:chebyshev_bound} on
$C=A_{\rho-k}$. Hence by applying a union bound on
Lemma~\ref{lem:sigma_bound}, Lemma~\ref{lem:chebyshev_bound} and
Inq.~\eqref{ineq:mm_frob}, we get that w.p. at least $0.97$,
\begin{eqnarray*}
\frobnorm{ E }  & \leq & \frobnorm{ A_{\rho-k}R R^\top  V_k  } + \frobnorm{ A_{\rho-k}R} \norm{ \pinv{(V_k^\top R)} - (V_k^\top R)^\top } \\
                           & \leq &  0.5 \eps \frobnorm{ A_{\rho-k}} + \sqrt{(1+\eps)} \frobnorm{A_{\rho-k}} \cdot 3 \eps \\
                           & \leq &  0.5 \eps \frobnorm{ A_{\rho-k}} +  3.5 \eps \frobnorm{A_{\rho-k}} \\
                           &  = &  4 \eps\cdot \frobnorm{A_{\rho-k}}.
\end{eqnarray*}
The last inequality holds thanks to our choice of $\eps \in
(0,1/3)$.
\end{proof}
\begin{proposition} \label{prop1}
A well-known property connects the SVD of a matrix and $k$-means
clustering. Recall Definition \ref{def:kmeans}, and notice that
$X_{opt}X_{opt}^\top A$ is a matrix of rank at most $k$. From the
SVD optimality we immediately get that 
\begin{equation}\label{eqn:svdcon}
\frobnorm{ A_{\rho-k} }^2\ =\ \frobnorm{ A - A_k }^2\ \leq\ \frobnorm{
A - X_{opt}X_{opt}^\top A }^2.
\end{equation}
\end{proposition}
\subsection{The proof of Eqn.~\eqref{eqn:main2} of Theorem~\ref{thm:second_result}} \label{sxn:proof3} 
%
%
%
We start by manipulating the term $\frobnorm{A -
X_{\tilde{\gamma}} X_{\tilde{\gamma}}^\top  A}^2$ in
Eqn.~\eqref{eqn:main2}. Replacing $A$ by $A_k + A_{\rho-k}$, and
using the Pythagorean theorem (the subspaces spanned by the
components $A_k - X_{\tilde{\gamma}} X_{\tilde{\gamma}}^\top A_k$
and $A_{\rho-k}- X_{\tilde{\gamma}} X_{\tilde{\gamma}}^\top
A_{\rho-k}$ are perpendicular) we get
\begin{eqnarray}\label{eqn:f1}
\frobnorm{A - X_{\tilde{\gamma}} X_{\tilde{\gamma}}^\top A}^2\ =\ \underbrace{\frobnorm{(I -
X_{\tilde{\gamma}} X_{\tilde{\gamma}}^\top ) A_k}^2}_{\theta_1^2}\ +\ \underbrace{\frobnorm{(I - X_{\tilde{\gamma}}
X_{\tilde{\gamma}}^\top )A_{\rho-k}}^2}_{\theta_2^2}.
\end{eqnarray}
We first bound the second term of Eqn.~\eqref{eqn:f1}. Since
$I-X_{\tilde{\gamma}}X_{\tilde{\gamma}}^\top $ is a projector
matrix, it can be dropped without increasing a unitarily invariant
norm. Now Proposition \ref{prop1} implies that
\begin{eqnarray} \label{eqn:f2}
\theta_2^2\ \leq\ \frobnorm{A_{\rho-k}}^2\ \leq\ \frobnorm{ A -
X_{opt}X_{opt}^\top A }^2.
\end{eqnarray}
We now bound the first term of Eqn.~\eqref{eqn:f1}:
\begin{eqnarray}
\label{t0} \theta_1
&\leq&  \frobnorm{(I -X_{\tilde{\gamma}}X_{\tilde{\gamma}}^\top )AR\pinv{(V_kR)}V_k^\top }\ +\ \frobnorm{E} \\
\label{t1}
&\leq&  \frobnorm{(I -X_{\tilde{\gamma}}X_{\tilde{\gamma}}^\top )AR} \norm{\pinv{(V_kR)}}\ +\ \frobnorm{E}\\
\label{t2}
&\leq& \sqrt{\gamma} \frobnorm{(I -X_{opt}X_{opt}^\top )AR} \norm{\pinv{(V_kR)}}\  +\ \frobnorm{E} \\
\label{t3}
&\leq& \sqrt{\gamma}  \sqrt{(1+\eps)} \frobnorm{(I -X_{opt}X_{opt}^\top )A}  \frac{1}{1-\eps}\ +\ 4 \eps \frobnorm{(I -X_{opt}X_{opt}^\top )A} \\
\label{t4}
&\leq& \sqrt{\gamma}  (1+2.5 \eps) \frobnorm{(I -X_{opt}X_{opt}^\top )A}\ + \sqrt{\gamma} \ 4 \eps \frobnorm{(I -X_{opt}X_{opt}^\top )A}\\
\label{t5}
&\leq& \sqrt{\gamma} ( 1 + 6.5 \eps) \frobnorm{(I
-X_{opt}X_{opt}^\top )A}
\end{eqnarray}
In Eqn.~\eqref{t0} we used Lemma \ref{lem:decomposition}, the
triangle inequality, and the fact that $I -
\tilde{X}_{\gamma}\tilde{X}_{\gamma}^\top $ is a projector matrix
and can be dropped without increasing a unitarily invariant norm.
In Eqn.~\eqref{t1} we used submultiplicativity (see Section
\ref{sxn:notation}) and the fact that $V_k^\top $ can be dropped
without changing the spectral norm. In Eqn.~\eqref{t2} we replaced
$X_{\tilde{\gamma}}$ by $X_{opt}$ and the factor $\sqrt{\gamma}$
appeared in the first term. To better understand this step, notice
that $X_{\tilde{\gamma}}$ gives a $\gamma$-approximation to the
optimal $k$-means clustering of the matrix $AR$, and any other $n
\times k$ indicator matrix (for example, the matrix $X_{opt}$)
satisfies
\begin{equation*}
\frobnorm{\left(I - X_{\tilde{\gamma}} X_{\tilde{\gamma}}^\top
\right) AR}^2 \leq\ \gamma\ \min_{X \in \cal{X}} \frobnorm{(I - X
X^\top ) AR}^2\ \leq \gamma \frobnorm{\left(I - X_{opt}
X_{opt}^\top \right) AR}^2.
\end{equation*}
In Eqn.~\eqref{t3} we used Lemma \ref{lem:chebyshev_bound} with $
C= (I - X_{opt} X_{opt}^\top )A$, Lemma \ref{lem:sigma_bound} and
Proposition \ref{prop1}. In Eqn.~\eqref{t4} we used the fact that
$\gamma \geq 1$ and that for any $\eps \in (0,1/3)$ it is
$(\sqrt{1+\eps})/(1-\eps) \leq 1+ 2.5\eps$. Taking squares in
Eqn.~\eqref{t5} we get
\[ \theta_1^2\ \leq\ \gamma ( 1 + 28 \eps ) \frobnorm{(I
-X_{opt}X_{opt}^\top )A}^2.\]
Finally, rescaling $\eps$ accordingly and applying the union bound on
Lemma \ref{lem:decomposition} and Definition \ref{def:approx} concludes the proof.
\section{Experiments}\label{sec:experiments}
This section describes an empirical evaluation of Algorithm \ref{alg:selection} on a face 
images collection. We implemented our algorithm in MatLab and compared it against other prominent dimensionality reduction techniques such as the Local Linear Embedding (LLE) algorithm and the Laplacian scores for feature selection. We ran all the experiments on a Mac machine with a dual core 2.26 Ghz processor and 4 GB of RAM.
Our empirical findings are very promising indicating that our algorithm and implementation could be very useful in real applications involving
clustering of large-scale data. 
\begin{figure}[tb]
\begin{center}
\includegraphics[width=0.33\textwidth]{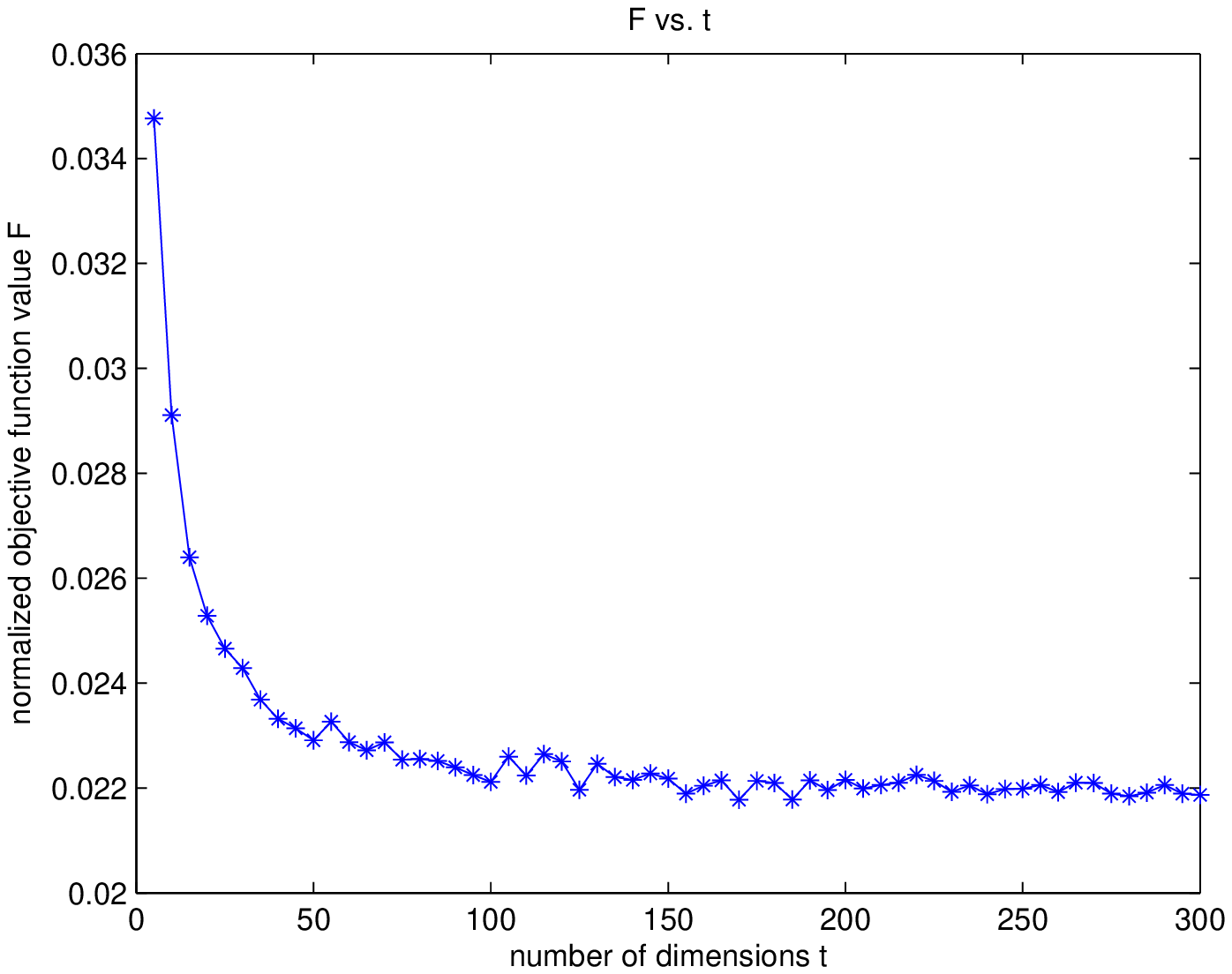}
\includegraphics[width=0.33\textwidth]{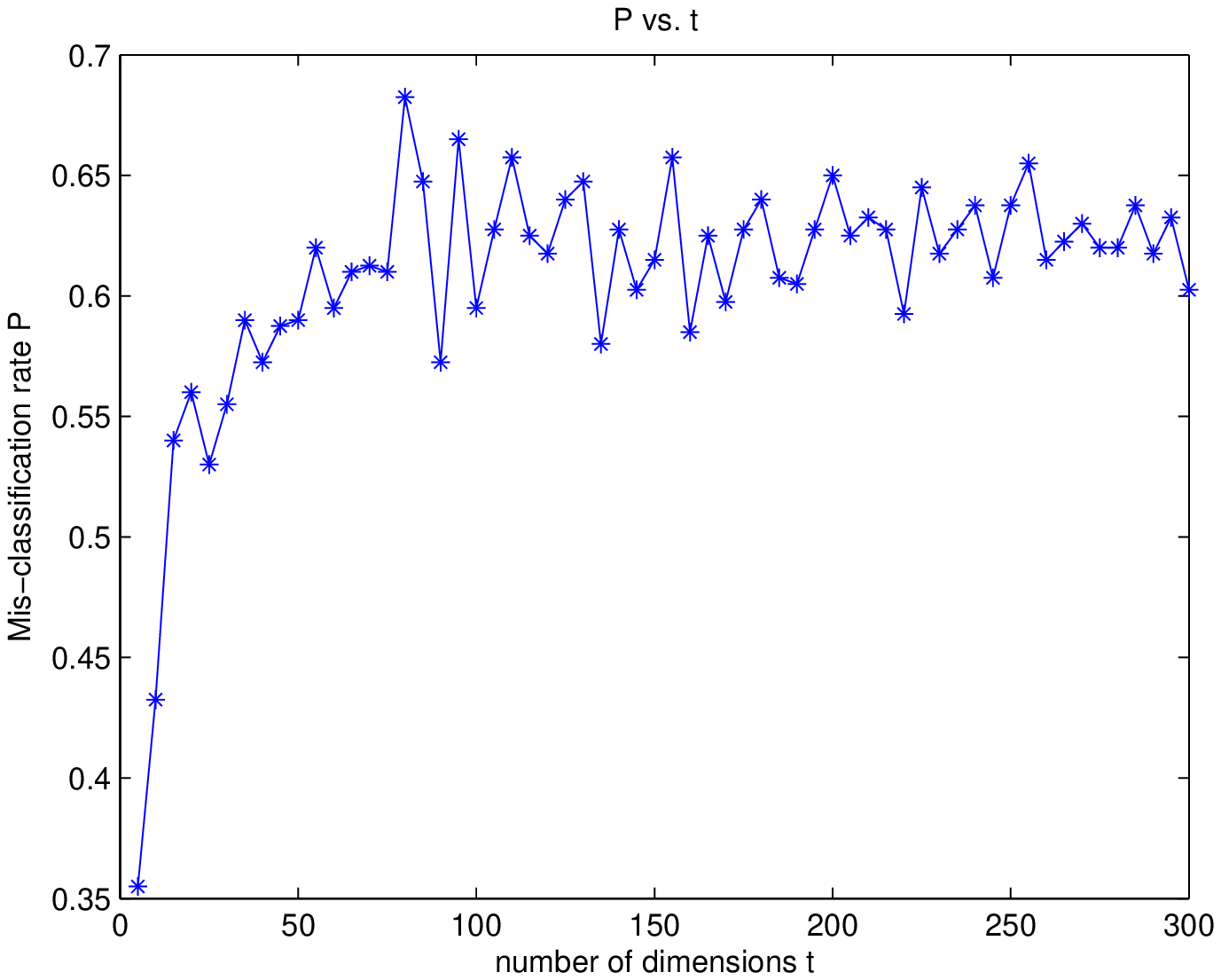}
\includegraphics[width=0.33\textwidth]{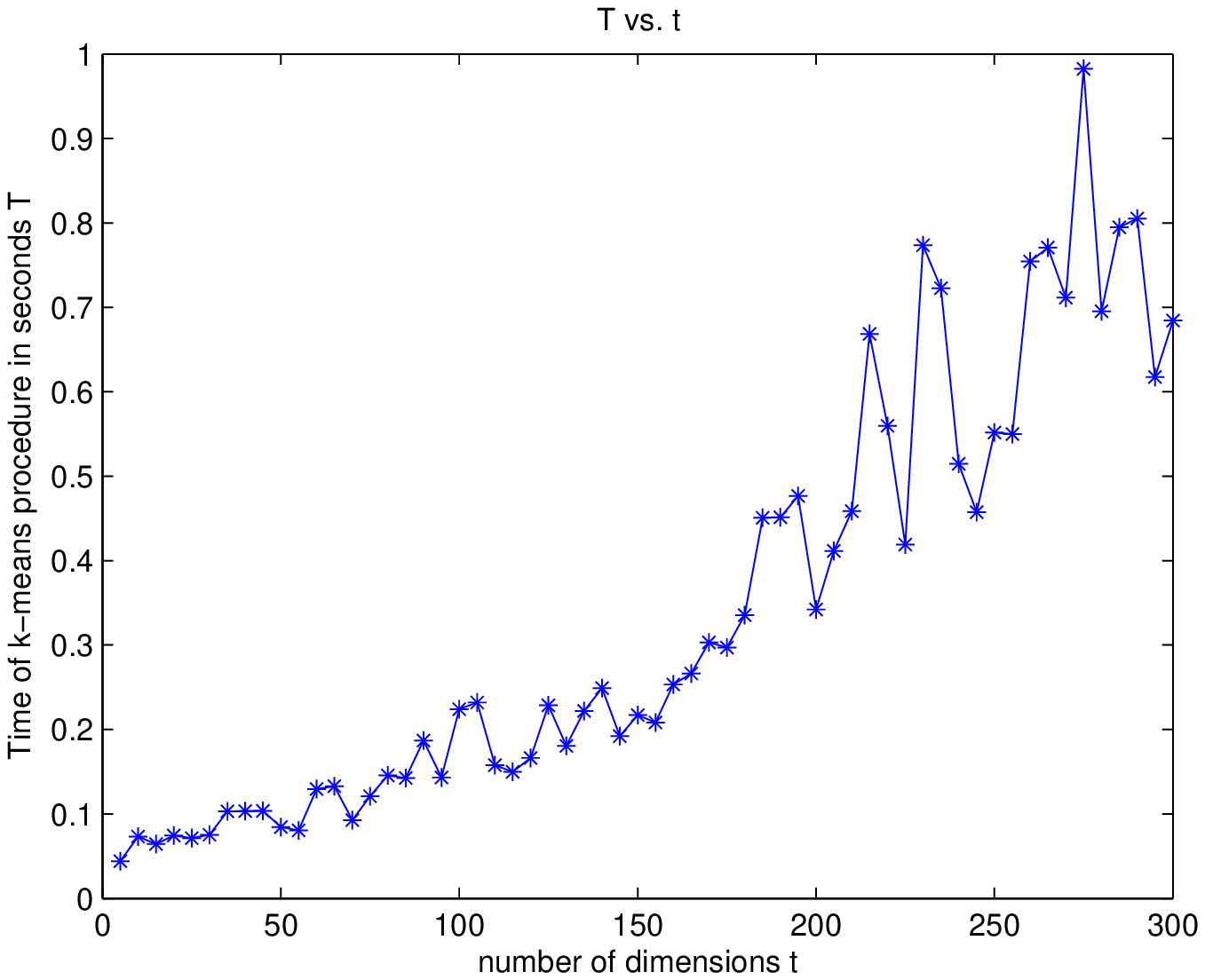}
\caption{The results of our experiments after running Algorithm 1 with $k=40$ on the face images collection.}
\label{fig:rp}
\end{center}
\end{figure}
\subsection{An application of Algorithm 1 on a face images collection}
We experiment with a face images collection. We downloaded the images
corresponding to the ORL database from \cite{www:data}. This collection
contains $400$ face images of dimensions $64 \times 64$ corresponding to $40$ different
people. These images form $40$ groups each one containing exactly $10$ 
different images of the same person.  After vectorizing each 2-D image and putting it as a row vector in an appropriate matrix, one can construct a $400 \times 4096$ image-by-pixel matrix $A$. In this matrix, objects are the face images 
of the ORL collection while features are the pixel values of the images. To apply the Lloyd's heuristic on $A$, 
we employ MatLab's function $kmeans$ with the parameter determining the maximum number of 
repetitions setting to $30$. We also
chose a deterministic initialization of the Lloyd's iterative E-M procedure,
i.e. whenever we call $kmeans$ with inputs a matrix 
$\tilde{A} \in R^{400 \times \tilde{d}}$, with $\tilde{d} \geq 1$, and the integer $k=40$, we 
initialize the cluster centers with the $1$-st, $11$-th,..., $391$-th rows of $\tilde{A}$, respectively.
Note that this initialization corresponds to picking images from the forty different 
groups of the available collection, since the images of every group are stored sequentially in $A$.
We evaluate the clustering outcome from two different perspectives. First, we
measure and report the objective function $F$ of the $k$-means clustering problem. In particular, we
report a normalized version of $F$, i.e. $\tilde{F} = F / || A ||_F^2$. Second, we report the mis-classification
accuracy of the clustering result. We denote this number by $P$ ($0 \leq P \leq 1$), where $P=0.9$, for
example, implies that $90 \%$ of the objects were assigned to the correct cluster after the application of the
clustering algorithm. In the sequel, we first perform experiments by running Algorithm \ref{alg:selection} with
everything fixed but $t$, which denotes the dimensionality of the projected data. Then, for four representative 
values of $t$, we compare Algorithm \ref{alg:selection} with three other dimensionality reduction methods as well with the approach of running the Lloyd's heuristic on the original high dimensional data. 

We run Algorithm \ref{alg:selection} with $t = 5, 10, ...,300$ and $k=40$ on the matrix $A$ described above. 
Figure \ref{fig:rp} depicts the 
results of our experiments. A few interesting observations are immediate. First, the normalized objective function
$\tilde{F}$ is a piece-wise non-increasing function of the number of dimensions $t$.  The decrease in $\tilde{F}$ is large in the first few choices of $t$; then, increasing the number of dimensions $t$ of the projected data decreases 
$\tilde{F}$ by a smaller value. The increase of $t$ seems to become irrelevant
after around $t=90$ dimensions. Second, the mis-classification rate $P$ is a
piece-wise non-decreasing function of $t$. The increase of $t$ seems to become
irrelevant again after around $t=90$ dimensions. Another interesting observation
of these two plots is that the mis-classification rate is not directly relevant
to the objective function $F$.  Notice, for example, that the two have different
behavior from $t=20$ to $t=25$ dimensions. Finally, we report the running time
$T$ of the algorithm which includes only the clustering step. Notice that the
increase in the running time is  - almost - linear with the increase of $t$. The
non-linearities in the plot are due to the fact that the number of iterations
that are necessary to guarantee convergence of the Lloyd's method are different for different values of $t$.
This observation indicates that small values of $t$ result to significant
computational savings, especially when $n$ is large. Compare, for example, the
one second running time that is needed to solve the $k$-means problem when $t=275$ against the $10$
seconds that are necessary to solve the problem on the high dimensional data. To
our benefit, in this case, the multiplication $A R$ takes only $0.1$ seconds
resulting to a total running time of $1.1$ seconds which corresponds to an
almost $90 \%$ speedup of the overall procedure.  

\begin{table}
\begin{center}
\begin{tabular}{|l|}
\hline
\multicolumn{1}{|l|}{} \\
\hline
\\
\hline
\textbf{SVD}  \\
\hline
\textbf{LLE}  \\
\hline
\textbf{LS}  \\
\hline
\textbf{HD}  \\
\hline
\textbf{RP}  \\
\hline
\end{tabular}
\begin{tabular}{|c|c|c|}
\hline
\multicolumn{2}{|c|}{\textbf{t = 10}} \\
\hline
$P$              &  $F$       \\
\hline
0.5900      &  0.0262 \\
\hline
0.6500      &  0.0245 \\
\hline
0.3400      &  0.0380 \\
\hline
0.6255      &  0.0220 \\
\hline
0.4225      &  0.0283 \\
\hline
\end{tabular}
\begin{tabular}{|c|c|c|}
\hline
\multicolumn{2}{|c|}{\textbf{t = 20}} \\
\hline
$P$              &  $F$       \\
\hline
0.6750      &  0.0268 \\
\hline
0.7125      &  0.0247 \\
\hline
0.3875      &  0.0362 \\
\hline
0.6255      &  0.0220 \\
\hline
0.4800      &  0.0255 \\
\hline
\end{tabular}
\begin{tabular}{|c|c|c|}
\hline
\multicolumn{2}{|c|}{\textbf{t = 50}} \\
\hline
$P$              &  $F$       \\
\hline
0.7650      &  0.0269 \\
\hline
0.7725      &  0.0258 \\
\hline
0.4575      &  0.0319 \\
\hline
0.6255      &  0.0220 \\
\hline
0.6425      &  0.0234 \\
\hline
\end{tabular}
\begin{tabular}{|c|c|c|}
\hline
\multicolumn{2}{|c|}{\textbf{t = 100}} \\
\hline
$P$              &  $F$       \\
\hline
0.6500      &  0.0324 \\
\hline
0.6150      &  0.0337 \\
\hline
0.4850      &  0.0278 \\
\hline
0.6255      &  0.0220 \\
\hline
0.6575      &  0.0219 \\
\hline
\end{tabular}
\caption{Numerics from our experiments with five different methods.}
\label{table:rp}
\end{center}
\end{table}

We now compare our algorithm against other dimensionality reduction techniques. In particular, in this paragraph we present head-to-head comparisons for the following five methods:
(i) SVD: the Singular Value Decomposition (or Principal Components Analysis) dimensionality reduction approach - we use MatLab's $svds$ function;
(ii) LLE: the famous Local Linear Embedding algorithm of \cite{LLE} - we use the MatLab code from \cite{www:lle} with the parameter $K$ determining the number of neighbors setting equal to $40$;
(iii) LS: the Laplacian score feature selection method of \cite{HCN06} - we use the MatLab code from \cite{www:soft} with the default parameters\footnote{In particular, we run $W = constructW(A);$ $Scores = LaplacianScore(A, W);$ }; 
(v) HD: we run the $k$-means algorithm on the High Dimensional data; and
(vi) RP: the random projection method we proposed in this work - we use our own MatLab implementation. 
The results of our experiments on $A$, $k=40$ and $t = 10, 20, 50, 100$  are shown in Table \ref{table:rp}. 
In terms of computational complexity, for example $t=50$, the time (in seconds) needed for all five methods (only the dimension reduction step) are $T_{SVD} = 5.9$, $T_{LLE} = 4.4$, $T_{LS} = 0.32$, $T_{HD} = 0$, and $T_{RP} = 0.03$. 
Notice that our algorithm is much faster than the other approaches while achieving worse ($t=10,20$), slightly worse ($t=50$) or slightly better ($t=100$)  approximation accuracy results. 
\subsection{A note on the mailman algorithm for matrix-matrix and matrix-vector multiplication}
\label{sec:mail}
\vspace{-0.15in}
In this section, we compare three different implementations of the third step of Algorithm~\ref{alg:selection}. As we already discussed in Section~\ref{sec:rt}, the mailman algorithm is asymptotically faster than naively multiplying the two matrices $A$ and $R$. In this section we want to understand whether this asymptotic 
behavior of the mailman algorithm is indeed achieved in a practical implementation. We compare three different approaches for the implementation of the third step of our algorithm: the first is MatLab's  function $times(A,R)$ (MM1);
 the second exploits the fact that we do not need to explicitly store the whole matrix $R$, and that the computation can be performed on the 
fly (column-by-column) (MM2); the last is the mailman algorithm~\cite{LZ09} (see Section~\ref{sec:rt} for more details). We implemented the last
 two algorithms in C using MatLab's MEX technology. We observed that when $A$ is a vector $(n=1)$, then the mailman algorithm is indeed faster than
 (MM1) and (MM2) as it is also observed in the numerical experiments of~\cite{LZ09}. Moreover, it's worth-noting that (MM2) is also superior compared to (MM1). On the other hand, 
our best implementation of the mailman algorithm for matrix-matrix operations is inferior to both (MM1) and (MM2) for any $ 10 \leq n \leq 10,000$. Based on these findings, we chose to use (MM1) for our experimental evaluations.

\textbf{Acknowledgments:} Christos Boutsidis was supported by NSF CCF 0916415 and a Gerondelis Foundation Fellowship; Petros Drineas was partially supported by an NSF CAREER Award and NSF CCF 0916415.



\clearpage

\bibliographystyle{abbrv}

\end{document}